\documentclass{article}

\usepackage{lmodern}  

\usepackage[colorlinks,pagebackref,linktocpage]{hyperref}

\hypersetup{pdfpagemode=UseNone}  

\usepackage{graphicx}

\usepackage{amssymb}  
\usepackage{amsmath}   
\usepackage{xspace}   

\usepackage{amsthm} 
\newtheorem{theorem}{Theorem}
\newtheorem{proposition}[theorem]{Proposition}

\newtheorem{lemma}[theorem]{Lemma}

\usepackage{amsopn}
\DeclareMathOperator*{\argmax}{arg\,max}
\DeclareMathOperator*{\argmin}{arg\,min}

\usepackage{color}
\definecolor{orange}{RGB}{255,127,0}
\definecolor{brown}{RGB}{150,70,0}
\definecolor{green}{RGB}{127,255,127}
\definecolor{darkgreen}{RGB}{0,127,0}
\definecolor{blue}{RGB}{127,127,255}
\definecolor{lightblue}{RGB}{150,150,255}
\definecolor{darkblue}{RGB}{0,0,127}
\definecolor{red}{RGB}{255,90,90}
\definecolor{violet}{RGB}{200,110,170}
\definecolor{grey}{RGB}{127,127,127}
\definecolor{pink}{RGB}{255,180,180}

\newcommand{\sidenoteJose}[1]{\textcolor{violet}{${\leftarrow}\hspace{-3pt}{\bullet}$}\marginpar{\textcolor{violet}{${\leftarrow}\hspace{-3pt}{\bullet}$}\\\tiny{\textcolor{violet}{#1}}}}

\newcommand{\BTD}{\ensuremath{\mathit{BTD}}\xspace}
\newcommand{\Kt}{\ensuremath{\mathit{Kt}}\xspace}

\newcommand{\length}{\ensuremath{\ell}\xspace}
\newcommand{\runtime}{\ensuremath{\tau}\xspace}

\newcommand{\represents}{\ensuremath{\triangleright}\xspace}
\newcommand{\Expected}{\ensuremath{\mathbb{E}}\xspace}

\newcommand{\comment}[1]{}

\usepackage{graphicx}

\title{Finite Biased Teaching \\with Infinite Concept Classes}
\author
{
	Jos\'{e} Hern\'{a}ndez-Orallo\\
	{\normalsize\em DSIC, Universitat Polit\`ecnica de Val\`encia, Spain}\\
	{\normalsize \tt jorallo@dsic.upv.es} \\	
	{\small ORCID: \url{http://orcid.org/0000-0001-9746-7632}}\\
	\\
	Jan Arne Telle\\
	{\normalsize\em Department of Informatics, University of Bergen, Norway}\\
	{\normalsize \tt Jan.Arne.Telle@uib.no} \\ 
	{\small ORCID: \url{http://orcid.org/0000-0002-9429-5377}}\\
}
\date{\today}

\begin{document}

\maketitle

\begin{abstract} 
We investigate the teaching of infinite concept classes through the effect of the {\em learning} bias (which is used by the learner to prefer some concepts over others and by the teacher to devise the teaching examples) and the {\em sampling} bias (which determines how the concepts are sampled from the class). We analyse two important classes: Turing machines and finite-state machines. We derive bounds for the {\em biased teaching dimension} when the learning bias is derived from a complexity measure (Kolmogorov complexity and minimal number of states respectively) and analyse the sampling distributions that lead to finite expected biased teaching dimensions. We highlight the existing trade-off between the bound and the representativeness of the sample, and its implications for the understanding of what teaching rich concepts to machines entails. 

\end{abstract}

\section{Introduction}

Learning from examples when the concept class is rich and infinite is usually considered a very hard computational problem. Positive results in theory and practice usually assume an infinite but not very expressible class,
or a strong bias, usually as a prior distribution over the concept class. A uniform choice for this distribution for discrete concept classes leads to zero probabilities or, worse, to no-free lunch results \cite{wolpert1996lack,wolpert1997no}. Consequently, other biases are usually assumed, either related to the application problem at hand or based on some notion of resources used by the  concepts. However, even with the use of strong biases, current machine learning techniques, and especially deep learning and reinforcement learning approaches, require a large amount of examples. 

Aware of this limitation, there has been a renewed interest in {\em teaching} computers \cite{khan2011humans,Zhu2013,zhu2015machine}, rather than just focusing on machine learning systems that can only expect examples at random. 
One of the key concepts in machine teaching is the power of choosing an optimal witness set \cite{freivalds1989inductive,shinohara1991teachability,freivalds1993power,goldman1995complexity}. 
This set is chosen as small as possible, such that the learner still identifies the concept. However, for interesting, rich concept classes we do not know how to choose just a few examples that, on expectation, make an existing learning system find the solution. This contrasts strongly with the way humans teach other humans, where even very complex Turing-complete (universal) concept classes in natural language can be transmitted using just a few examples. For instance, when humans are said that 
``dollars'', ``euros'' and ``yens'' are positive examples but ``deutschemarks'' are not, 
most 
understand that the concept is about currencies that are legal tender today. 
This kind of learning (or {\em teaching}, where the examples for the concepts are chosen, as with these words), is still beyond current technology ---and we do not fully understand why. This is also related to natural language understanding, and the fact that humans often transmit concepts by example, rather than using the description of the concept, another poorly-understood phenomenon that requires strong biases on sender and receiver \cite{SHAFTO201455}. 

The teaching dimension of a concept \cite{freivalds1989inductive,shinohara1991teachability} in some concept class is the minimum number of examples required such that a learner uniquely identifies (learns) the concept. The teaching dimension of a concept {\em class} is usually understood as the worst case, which is usually unbounded for many infinite concept classes. With the use of preferences (a kind of bias) we get some finite (worst-case) teaching dimensions for some restricted languages \cite{gao2016preference}. 
However, for richer languages, can we get finite, and even short, teaching dimensions {\em on average}? 
A uniform distribution, usually assumed for finite classes \cite{anthony1995specifying,lee2007dnf}, cannot be applied to infinite concept classes. 
The main insight comes if we realise that there are 
two kinds of bias: the {\em learning} bias and the {\em sampling} bias. 

The learning bias makes the learner prefer some concepts over others. If the given witness set is consistent with (infinitely) many concepts, the one that is preferred will be output. This preference can also be updated as more examples are seen, adjusting the posterior probabilities. 

The sampling bias is used by the teacher (or tester) to see whether the learner is able to learn the whole class and not just a particular subset of it. Consequently, it has to be as diverse (entropic) as possible. Note that the sampling bias is about a representative choice of concepts, not about the intentional 
choice of the examples for each concept. 

Both biases are referring to how likely or expectable a concept is, and should be linked in some way. 
Indeed, we investigate whether this alignment between the learning bias (`chosen' by both learner and teacher) and the sampling bias (perhaps fixed or chosen by a tester) can lead to short example sets on average, ensuring that teaching sessions are feasible.


Of course, one can always get a finite expected teaching dimension by putting almost all the mass of the distribution on a few concepts. The question is whether there are some reasonable biases, still with infinite Shannon entropy \cite{Tadaki}, for which teaching is feasible. The main observation is that the more expressive the language is the more extreme (biased) the distributions must be in order to get teachability, but the distributions can still be sufficiently entropic at one end. This view creates a relation between the expressiveness of a language and how entropic the bias must be in order to make teaching possible, a more gradual alternative to the traditional (Chomskian) hierarchical view of languages. 

In this paper, we analyse biases that are derived from complexity functions (program length, number of states, running times, etc.). This leads to the interpretation that if concept $c_1$ is simpler than $c_2$ then it is preferred by the learner given the same witness set, and it will be more likely to be sampled by the teacher. This also implies that if a learner has a bias, its representation language should be aligned with it, making more likely concepts require fewer resources in the language (as it happens with human language and, of course, in communication theory). 

Given this new notion of expected biased teaching dimension (\BTD), we obtain two major results. First, we get finite (and actually small) expected values for Turing-complete languages. This is in alignment with the observation of humans requiring very few examples when teaching or transmitting concepts in natural language. Second, we derive effective settings for a particularly interesting infinite concept class, the set of regular languages. More precisely, we provide a series of contributions: 

\begin{itemize}
\item We show that teaching for rich infinite concept classes not only requires that some common bias is shared by learner and teacher (the learning bias), but also that testing actual teaching for the whole class is as representative as possible (the sampling bias). 
\item We present a new conceptualisation of expected \BTD using the learning and the sampling bias.
\item We provide results showing that the expected \BTD for Turing-complete languages is small, with the universal biases based on the program size of the concepts.
\item Since universal biases based on Kolmogorov complexity are incomputable, we introduce computational time in the measure of concept complexity. We show that the learner becomes computable but the teacher does not. 
\item We show finite expected \BTD for regular languages using biases derived from the number of states of the minimal finite state machine (FSM) expressing the concept, proving both learner and teacher are computable. 
\end{itemize}


\section{Teaching sets for infinite concept classes: learning bias}\label{sec:learning}

Let us first introduce the classical teaching dimension. 
We have a possibly infinite instance space $X$, with instances $x_i \in X$, 
that can be either positive examples, denoted by a pair $\left\langle x_i, 1\right\rangle$, usually represented as $x_i^+$, or negative examples, denoted by a pair $\left\langle x_i, 0\right\rangle$, usually represented as $x_i^-$. 
A concept is a binary function over $X$ to the set $\{0,1\}$. 
A concept language or class $C$ is composed of a possibly infinite number of concepts. 
An example set $S$ is just a (possibly empty) set of examples. We say that a concept $c$ satisfies (or is consistent with) $S$, denoted by
$c \vDash S$,
if $c(x_i)=1$ for the positive examples in $S$, and $c(x_i)=0$ for the negative ones.
All concepts satisfy the empty set. Given this, the teaching dimension (TD) of a concept $c$ can be defined as follows:
\[ TD(c) \triangleq \min_S \{ |S| : \{ c \} = \{c' \in C : c' \vDash S \} \} \]
%
%
This minimal set is known as a witness set, and the teacher can assume that the learner will infer the concept given its witness set. 
%
Some further assumptions are needed. For instance, one can define ``coding tricks'' \cite{balbach2007models,balbach2009recent}, such as assuming a coding between instances and concepts, so that the $j^{th}$ instance always corresponds to the $j^{th}$ concept, so basically one only needs to send the ``index'' to identify the concept, as a lookup table. 
An appropriate way \cite{goldman1993teaching} to prevent this  
considers that whenever a learner identifies a concept $c$ with an example set $S$, it must also identify $c$ with any other superset of $S$ that is also consistent with $c$. The Recursive Teaching Dimension (RTD) \cite{zilles2011models,Chen2016} is a variant where concepts are taught with an order, starting for those of smallest dimension and removing the identified concepts for the following iteration. This 
becomes slightly more powerful than the classical teaching dimension but still compatible with Goldman and Mathias's condition. 
Additionally, RTD 
is related to the VC dimension, see e.g. \cite{MoranSWY15,Chen2016}.

One thing to note about these settings is that extra examples (further confirming evidence) will not change the certainty of the learner about the concept. 
However, both machine teaching and learning are inductive processes where the reliability of a hypothesis can increase with confirming data 
by discarding alternative hypotheses. 
In other words, the classical teaching dimension (and the PBTD we will mention below) is more about identification rather than inductive inference. But the learner should be increasing its confidence as it gets more examples, even past the identification.

We can reconcile this by considering that the learner has a prior, and as more examples are seen, more hypotheses are excluded, but at the same time the posterior of the remaining hypotheses is changing. 
%
%
%
%
%
%
In order to do this, we now define a bias as a probability distribution $w(c)$ over $C$, which represents the {\em learning bias}. 
Using this bias $w$, we can define the biased teaching dimension:
\begin{equation}\label{eq:BTD}
\BTD_w(c) \triangleq \min_S \{ |S| : \{ c \} = \argmax_{c' \vDash S} \{ w(c') \} \} 
\end{equation}
Basically the bias $w$ introduces a preference when choosing among consistent hypotheses. This is an alternative formulation (quantitative, so necessarily a total order if concepts are arranged into batches of same $w$) to the preference-based teaching dimension (PBTD) \cite{gao2016preference}, and ultimately related to the $K$-dimension (\cite{balbach2007models,balbach2008measuring}), where this preference or ranking is linked to a measure of complexity, as we will revisit below. 
We can see that for every bias, the \BTD meets Goldman and Mathias's condition. 
We also see explicitly that the classical teaching dimension is assuming that all concepts are equally likely (maximum entropy), which is unrealistic in many situations (and if extended to infinite concept classes would lead to the no-free-lunch theorems \cite{wolpert1996lack,wolpert1997no}). 


Using a probabilistically biased interpretation of the teaching dimension, 
we can define a normalisation term as the overall a priori distribution mass of the consistent concepts so far, given a set $S$: 
$m_w(S) \triangleq  \sum_{c' \vDash S} w(c')$. 
%
The posterior gives a probabilistic assessment for a concept after seeing $S$, namely $w(c|S) = w(c) / m_w(S)$ if $c \vDash S$ and 0 otherwise.


\begin{table}
\begin{center}
\begin{small}
\tabcolsep=0.11cm
\begin{tabular}{cccccccccccc}
& $x_1$ & $x_2$ & $x_3$ & $x_4$ & $x_5$ & $x_6$ & $x_7$ & ... & $w(c_i)$ & $TD$ & $\BTD_w$ \\ \hline
$c_1$ & 0 & 0 & 1 & 0 & 1 & 1 & 0 & ... & 0.30  & $\infty$ & 0 \\
$c_2$ & 0 & 1 & 0 & 1 & 1 & 1 & 0 & ... & 0.25 & $\infty$ & 1 \\
$c_3$ & 1 & 0 & 0 & 1 & 1 & 1 & 0 & ... & 0.20  & $\infty$ & 1 \\
$c_4$ & 0 & 0 & 0 & 0 & 1 & 1 & 0 & ... & 0.05 & $\infty$ & 2 \\
$c_5$ & 0 & 0 & 0 & 0 & 0 & 1 & 0 & ... & 0.01 & $\infty$ & 1 \\
$c_6$ & 0 & 0 & 0 & 1 & 1 & 0 & 1 & ... & 0.01 & $\infty$ & 1 \\
Rest  & - & - & - & - & - & - & - & ... & 0.18 & . & . \\ \hline
\end{tabular}
\vspace{-0.2cm}
\caption{An infinite concept class with a learning bias $w$ where the six most likely concepts only differ on seven examples. The `Rest' row captures all the other concepts.}\label{tab:BTD2}
\vspace{-0.5cm}
\end{small}
\end{center}
\end{table}

Interestingly, this now becomes a truly inductive process. 
For the concept class in Table~\ref{tab:BTD2}, when no example is given, 
$m_w(\emptyset) = 1$. The posteriors are still equal to the priors (e.g., the probability for $c_4$ is still 0.05). 
If $x_4^-$ is presented, then we can discard $c_2$, $c_3$, $c_6$ and perhaps some other concepts. Imagine that half of the concepts in `Rest' are discarded. This would lead to $m_w(\{x_4^-\}) = 0.36 + 0.09 = 0.45$ with the posterior probability for $c_4$ being now 0.05/0.45 = 0.11 (but not the highest of the compatible concepts, which is still $c_1$). If $x_3^-$ is added to the set, then $c_1$ is now found inconsistent, and assuming that two thirds of the remaining concepts in `Rest' are discarded, we would have $m_w(\{x_4^-,x_3^-\}) = 0.06 + 0.03 = 0.09$ with the posterior probability for $c_4$ being updated to 0.05/0.09 = 0.54. This is now the highest (note that we only need to look at Rest to recalculate the probabilities, but not to know that this is the highest). 
 We now see that $\BTD_w(c_4)$ is not higher than 2, and since no single example can distinguish it from $c_1, c_2, c_3$, it is actually 2. 
Note that this concept $c_4$ can be suggested by the learner even if it is not the only compatible concept. Finally, if $x_5^+$ is shown and discards one third of the remaining in ``Rest'', then $m_w(\{x_4^-,x_3^-,x_5^+\}) = 0.05 + 0.02 = 0.07$ and the posterior probability for $c_4$ will now be 0.05/0.07 = 0.71. 
We see that with \BTD, the posterior probabilities can still increase when receiving further consistent evidence.

\comment{
In other words, 
the teaching dimension is just a particular point in this inductive progression when one hypothesis is singled out. For a uniform bias this inductive inference progression is reduced to a more abrupt change (the last step would be between a normalised $w(c) \leq 0.5$ with 2 or more alternative hypotheses to $w(c)= 1$ when only one remains), and once the identification is performed, increasing consistent evidence will not change the probabilities. For non-uniform (less entropic) bias the identification is not the end of the road and confirmining examples will increase the certainty about the chosen concept. Wether this stops or not depends on the concept class and the bias (for Turing-complete languages there are always alternative concepts for every set and we will never reach complete certainty).  
}

\section{Finite expected biased teaching dimension: sampling bias }\label{sec:sampling}

Up to this point, we have talked about the teaching dimension of one concept in a class. 
The teaching dimension of the whole class, and the classical worst-case scenario is defined 
%
$\max_{c \in C} \BTD_w(c)$. 
%
%
For many infinite concept classes, even with the use of a strong learning bias, there will not be an upper bound on the number of examples needed to distinguish the concepts. So, it becomes necessary to talk about an expected \BTD for a concept class $C$. This introduces a {\em sampling probability} over concepts $v$, which is used to obtain the expected \BTD for a concept class.
\begin{equation}\label{eq:expBTD}
 \Expected_v [\BTD_w(C)] \triangleq \sum_{c \in C} v(c) \cdot \BTD_w(c) 
\end{equation}
Of course, the result will strongly depend on the choice of $v$. One possible option is to assume $v(c) = w(c)$, meaning that the  probability that is used for calculating the plausibility of a concept (the learning bias) is also the same for the probability of that concept to appear (the sampling bias). 
%
%
%

The key question comes with rich concept classes with infinitely many concepts and, as a result, infinitely many examples (otherwise some concepts would not be distinguishable by definition). We cannot choose a uniform distribution for neither $w$ nor $v$ if the class is infinite and discrete. 

A natural idea when assigning a non-zero probability to an infinite discrete set of concepts is to use some distribution that is inversely related to the resources or complexity required by the concept, as given by a complexity function $K: C \rightarrow \mathbb{N}$ assigning a complexity value $k$ for all concepts. This is actually the idea behind the $K$-dimension \cite{balbach2007models,balbach2008measuring}. However, we now need to apply this to the sampling distribution as well in order to calculate the expected biased teaching dimension. 
First, we assume that the learning bias is consistent with the complexity function, i.e., inversely monotonically related: 
\begin{equation}\label{eq:learningmonotonicity}
\forall c_1, c_2 \in C : w(c_1)  \geq w(c_2) \Leftrightarrow K(c_1) \leq K(c_2) 
\end{equation}
\noindent From the infinitely many sampling distributions $v$, it makes sense to choose a distribution that is compatible with the learning distribution:
%
%
\begin{equation}\label{eq:samplingmonotonicity}
\forall c_1, c_2 \in C : v(c_1)  \geq v(c_2) \Leftrightarrow w(c_1) \geq w(c_2)
\end{equation}
which, from Eq.~\ref{eq:learningmonotonicity}, implies that both distributions are monotonically related. 
Let us denote by $C_k$ the ``batch'' composed of all the concepts of complexity $k$, i.e., $C_k = \{c : K(c) = k \}$,  which, from Eq.~\ref{eq:learningmonotonicity}, means that $w$ and $v$ are constant in each batch. The size of each batch is $N_k=|C_k|$. 

Then we add up all the sampling probabilities of the same batch, denoted by $V_k= \sum_{c \in C_k} v(c)$. 
The expected \BTD then becomes: 
%
\[ \Expected_v [\BTD_w(C)] = \sum_{k=1}^{\infty} \frac{V_k}{N_k} \sum_{c \in C_k} \BTD_w(c) \]
The average $\BTD_w$ for a batch $k$ is given by $\frac{1}{N_k}\sum_{c \in C_k} \BTD_w(c)$. 
Let us now consider that we have an upper bound for this average, denoted by $D_k$.  Then, 
\begin{equation}\label{eq:batches}
\Expected_v [\BTD_w(C)] \leq \sum_{k=1}^{\infty} V_k \cdot  D_k 
\end{equation}
This means that once the batches are created by the complexity function, the expected \BTD only depends on the progression of the sampling distribution by batches and the progression of (a bound of) the average \BTD in the batch. 
%
%
Figures~\ref{fig:plot1} and ~\ref{fig:plot2} show an example where the batched sampling distribution is geometric with parameter $1/6$, i.e., $V_k = (1/6)\cdot(5/6)^{k-1}$ with upper bound on average \BTD in the batch of $D_k = k^2 $. With these parameters, the sum converges to a finite expected \BTD: 66.


\begin{figure}[ht]
	\begin{center}
  \includegraphics[width=0.36\textwidth]{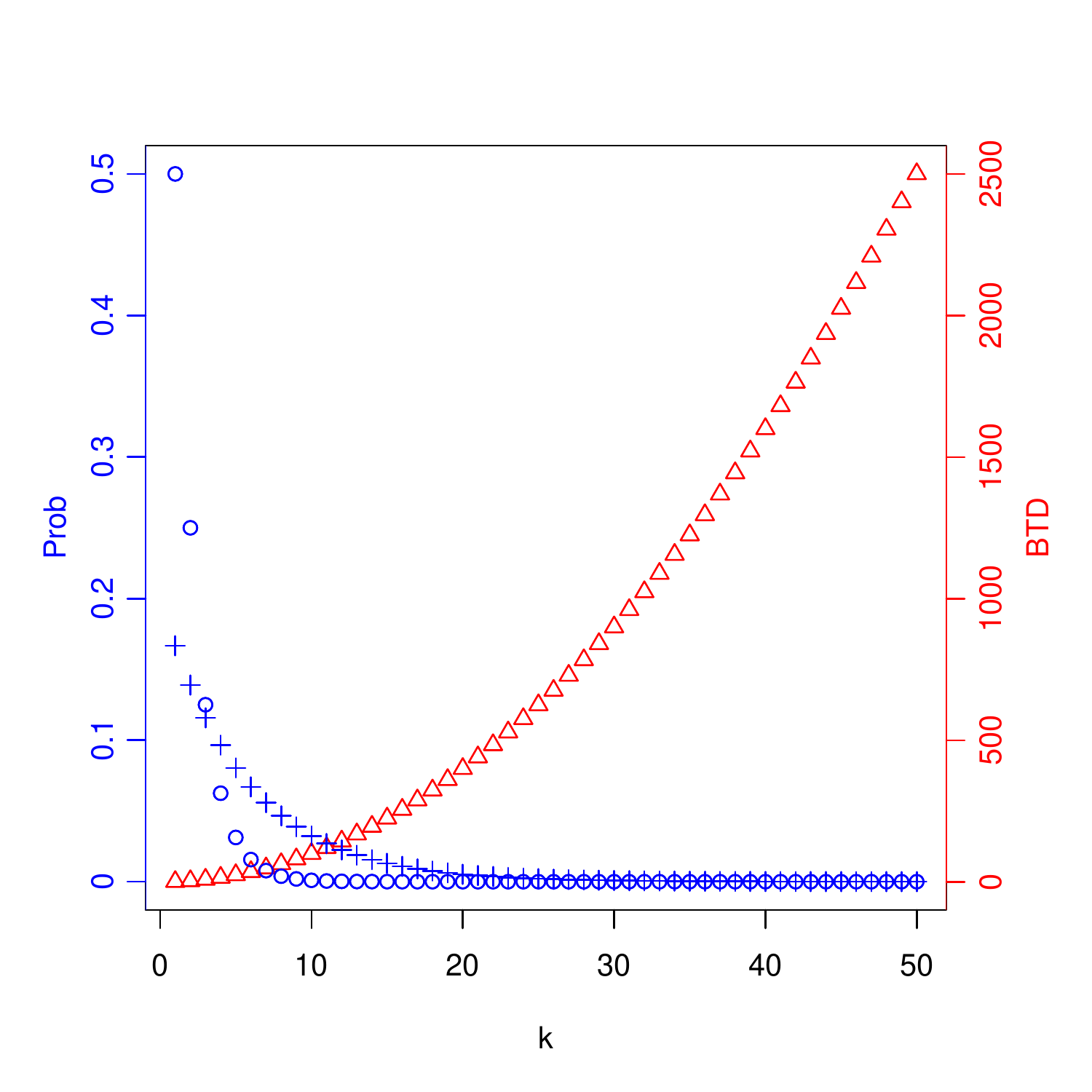} 
	\end{center}
	\vspace{-0.7cm}
  \caption{The different components in the expected \BTD. The summed sampling bias $V_k$ (blue crosses) for each batch $k$, and also the summed learning bias (blue circles). Also (red triangles) the (bound of the) average \BTD per batch $k$. }
	\label{fig:plot1}
\end{figure}

\begin{figure}[ht]
   \begin{center}
    \includegraphics[width=0.36\textwidth]{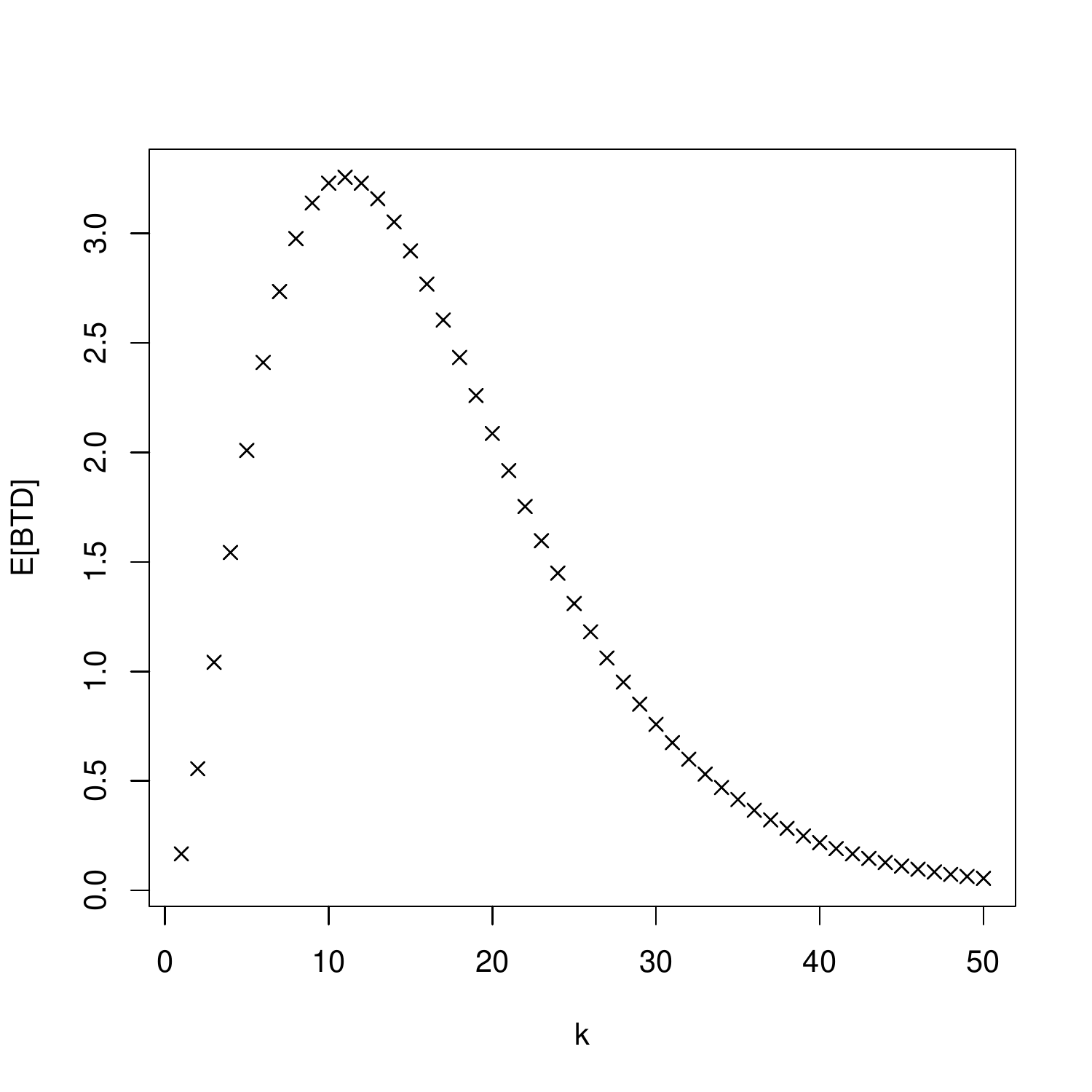} 
		\end{center}
		\vspace{-0.7cm}
  \caption{Components in the expected \BTD. The composition of the sampling bias with the \BTD gives the contribution of the expected \BTD for each value of $k$, whose sum in this case is finite (66).}
	\label{fig:plot2}
\end{figure}


The relevant question is, once we achieve a bound $D_k$, can we think of a sampling distribution that can guarantee a bounded number of examples in the teaching sets {\em on average}? 
Even with the constraint given by \ref{eq:samplingmonotonicity}, there are many distributions for $v$. One trivial case to minimise Eq.~\ref{eq:batches} is to choose $v$ in such a way that it gives all the mass of the probability to one batch with low or minimal teaching dimension. 
Basically this would restrict the class to a finite distorted version.  
Consequently, a trade-off emerges between $\Expected_v [\BTD_w(C)]$ and $v$. More entropic (or diverse) sampling distributions $v$ will be able to capture the breadth of the concept class (and actually be {\em representative} of it) at the cost of having higher expected \BTD. In any case, it is important to determine those distributions for which the expected \BTD is infinite, because, for those, teaching will be impossible. It is then the relation between the teaching dimension using a learning bias and the sampling distribution used for expectation what we investigate next, for two very important concept classes: Turing machines and finite-state machines.

\section{Expected BTD for 
universal languages (TMs)}\label{sec:TM}

Turing machines represent the most general class for (traditional) computation. Consequently, the choices for $w$ and $v$ will connect with fundamental computational concepts such as Kolmogorov complexity, Solomonoff's prediction and inductive inference in general \cite{LV2008}, giving us an overall view of the problem. For Turing machines, programs map to computable binary functions, as there are infinitely many for each concept. We say that a concept $c$ is represented by program $p$ in a universal Turing machine (UTM) $M$, denoted by $M(p) \represents c$, if for every example $\left\langle x_i, b\right\rangle$ in $c$ we get that the machine $M$, after being fed by the program $p$ and an appropriate binary encoding of the example (examples are natural numbers) 
outputs the correct label, i.e., $M(\left\langle p,\sigma_i \right\rangle)$ writes $b$ on the output string and halts. 
%
We now look for a measure of complexity of the concepts, so we extend the notion of Kolmogorov complexity as follows:
$K_M(c) = \min_{p: M(p) \represents c} \length(p)$,  
where $\length(p)$ is the length of $p$ in bits. In other words, the complexity of a concept is the shortest program that represents (computes) the concept. 
We now define $U_M(c) = 2^{-K_M(c)}$, which is a universal distribution over concepts based on their {\em algorithmic probability} \cite{LV2008}. To ensure that the sum is $\leq 1$, $M$ must be a prefix-free UTM. Still, since a concept can be represented by infinitely many programs, this $U_M$ will not add up to one, but it can be normalised to make an actual distribution $w$. 
 %
%
%
%
To highlight the dependency on the UTM chosen, we use notation $\BTD_M$ when $w=U_M$. We now can simplify Eq.~\ref{eq:BTD}:
%
%
%
\begin{align*}
BTD_M(c) &= \\ 
\min_S \{ |S| : \{ c \} \!=\! \argmax_{c' \vDash S} & \{ 2^{-\min_{p: M(p) \represents c'} \length(p)}  \} \} =  \\ 
			  \min_S \{ |S| : \{ c \} \!=\! & \argmin_{c' \vDash S, p: M(p) \represents c'} \length(p) \}  
\end{align*}						
\noindent 
We now have to look at the sampling distribution $v$. A common choice here is yet again a universal distribution $v(c) = 2^{-K_M(c)}$. This means that for each concept whose shortest program has size $k$ its probability is $2^{-k}$. The probability of all the concepts in the batch is then $V_k=2^{-k} \cdot N_k$. From here, 
%
we can instantiate Eq.~\ref{eq:expBTD} 
by batches as for Eq.~\ref{eq:batches}:
%
%
%
%
\begin{align}\label{eq:averagewithK}
\Expected_M [\BTD_M(C)] &= 
\sum_{c \in C} 2^{-K_M(c)} \cdot  BTD_M(c) \leq \nonumber \\
&\leq \sum_{k=1}^{\infty} 2^{-k} \cdot  N_k \cdot D_k  
\end{align}
%
%
%
%
%
%
%
\noindent The question is how we can bound the average teaching dimension for batch $k$. 
From Kushilevitz et al. \cite{kushilevitz1996witness} 
we know that for a finite concept class $C$ of binary vectors of length $m$ we have that the average teaching dimension (assuming uniform bias $u(c)=1/|C|$)), i.e., 
$\Expected [\BTD_u(C)]$, is bounded as follows:
%
%
\[ \forall C : \Expected [\BTD_u(C)] \leq 2\sqrt{|C|} \]
Interestingly, for batch $k$, we only need to distinguish a concept from all the other concepts in its batch $N_k$, and the concepts in previous batches. Let us denote by $N_{\leq k}$ the number of concepts in batches 1 to $k$. This means that the average \BTD for $C_k$ is bounded by 
$2\sqrt{N_{\leq k}}$.

But what is $N_k$?, i.e., how many concepts have shortest programs of size $k$? This cannot be $2^k$, since it has to be a prefix coding. The actual value will depend on the chosen coding. 
For instance, if we use a unary coding, we can get a convergent result very easily, since there is only one program for each $k$, so the term $N_k$ would be $1$ and the term $N_{\leq k}$ would be $k$. 
However, a unary coding is not universal.

We can try with Elias gamma coding \cite{elias1975universal,sayood2002lossless}. 
This is not asymptotically optimal, 
but it is still universal. Basically, this coding uses a leading sequence of $k$ zeros (which states the size of the string), followed by a 1 and then the traditional binary coding of a number. For instance, the first 10 codewords are 1, 010, 011, 00100, 00101, 00101, 00110, 00111, 0001000, 0001001, 0001010. As we can see, for each batch of the same size we have $2^i$ codewords with a size of $2i+1$, with $i$ being the index of batch starting at 0, and this gives an upper bound on $N_k$.  So now we have: 

\begin{proposition}\label{prop:Elias}
The expected biased teaching dimension assuming a universal distribution with an Elias gamma coding is finite, bounded by $1 + \sqrt{2}$. 
\end{proposition}

\begin{proof}
We have: 
\begin{eqnarray*}
\Expected [BTD_M(C)] & = & \sum_{k=1}^{\infty} 2^{-k} \cdot  N_k \cdot 2 \sqrt{N_{\leq k}}
\end{eqnarray*}
Since we have the correspondence $k=2i+1$, $N_k = 2^{i-1}$, $N_{\leq k} = 2^{i+1}-1$, and we have:
\begin{eqnarray*}
\Expected [BTD_M(C)] & \leq & \sum_{i=0}^{\infty} 2^{-(2i+1)} \cdot  2^{i-1} \cdot  2 \sqrt{2^{i+1}-1} \\
                     & \leq & 
										          \sum_{i=0}^{\infty} 2^{\frac{-i-1}{2}} = 1 + \sqrt{2} 
\end{eqnarray*}
\end{proof}

%
%
%
%
\noindent This means that with some universal codings we can have a finite expected \BTD. In other words, if a teacher samples concepts according to its universal distribution 
using a Elias gamma coding and both teacher and learner use the size of their programs as learning bias, then the number of examples needed to teach the concepts is finite on expectation. Of course, this is the case because the very small programs (and hence very simple concepts) dominate the distribution. However, we can modify the UTM and the coding in such a way that a more uniform-like distribution happens for sizes $k$ up to any arbitrary size $k_s$ provided that from that point on the distribution decays as fast as above.

The \BTD we have defined above is incomputable, since $K$ is incomputable. Can we think of a similar computable procedure to get a similar result? For instance, given a language $L$, a concept class $C$ and a concept $c$, 
the teacher should be able to compute the associated small teaching set $S$ and the learner should identify $c$ from it. 
%
To get a finite procedure we investigate the introduction of computational steps in the complexity function, inspired by Levin's $\Kt$ \cite{Le73,LV2008}, namely: 
\[ \Kt_M(p, S) \triangleq \length(p) + \log \sum_{s \in S} \runtime_M(p,s) \]
\noindent where $\runtime(p,s)$ represents the runtime of executing program $p$ on example $s$ to get a result. 


The original dovetail search of Levin's universal search is 2-dimensional on an increasing budget: over programs of increasing size and over increasing runtimes. Here, we add a third dimension: over increasing sizes of encodings of examples, to get the following results. 

\begin{proposition}\label{prop:Ktlearner}
Using $\Kt_M(p, S)$, for every $M$ and $c$, if given a minimal set $S$, a learner can identify $c$ by computable finite means.
\end{proposition}

\begin{proof}
The learner will follow a dove-tailing approach with an increasing budget. With budget $B$ on $\Kt_M$ the learner will enumerate all possible programs $p'$ 
ensuring that $\Kt_M(p', S') \leq B$. Note that this enumeration and its execution is finite because of the $\runtime$ term. 
For those programs inside the budget we discard those that are not consistent with the set.
Once the enumeration for a budget is exhausted, the budget is increased by 1. Ultimately, the first program that accepts the examples in $S$ in the budget will be found. The learner has identified the concept.
\end{proof}


\begin{proposition}\label{prop:Ktteacher}
Using $\Kt_M(p, S)$, given an $M$ and $c$, the generation of the minimal set $S$ by the teacher is incomputable.
\end{proposition}

\begin{proof} 
The teacher can try a dovetail enumeration but it has to check that all simplest concepts are different from $c$. And this is incomputable in general. More precisely, this can be seen by reduction from the undecidable predicate $Equiv(p1,p2)$, which tests equivalence of two programs (TMs). We have two algorithms, Learner and Teacher.
$Learner(S)=p$, where $p$ is the simplest program compatible with all pairs in $S$. 
$Teacher(p)=S$, where $S$ is the smallest set such that $Learner(S)=p'$, with $Equiv(p,p')$.  
We know $Learner$ is decidable. 
If $Teacher$ were decidable then we could decide $Equiv$ by
$Equiv(p1,p2)$ if and only if $Learner(Teacher(p1))=Learner(Teacher(p2))$. By contradiction, $Teacher$ is undecidable. 
\end{proof}

Even if the teacher knows the shortest program $p$ for a concept, there might be problems. For instance, if $p$ cannot be identified for a budget, for the next budget new programs may appear that compete with it (are compatible) on the examples under the budget. These alternative programs can be more efficient than $p$ (e.g., using partial lookup tables). This problem will appear for those programs whose time complexity increases exponentially (or even higher) in the size of the examples. There are possible solutions to be explored with bounded time or including the size of the proof to show that concepts are equal or not (so the class is reduced to Turing machines such that it can be proven or disproven equivalence to all simpler programs). We leave this as future work and focus on regular languages in the following section.

\section{Expected BTD for regular languages (FSMs)}\label{sec:FSM}

Regular languages are defined by finite state machines (FSMs), a very well-known class of concepts in computer science. One of the advantages of using FSMs, over TMs, is that some of the ingredients needed for an effective (and computable) teaching setting are present for FSMs. 
First, there is an algorithm with time complexity $O(k \log k)$ to reduce any FSM on $k$ states to an equivalent FSM on a minimum number of states \cite{hopcroft1971n}, and secondly there is an algorithm linear in the number of states to test equivalence of two FSMs \cite{hoka71}.  
As a concept is represented by its canonical FSM, the number of states $k$ can be used as a natural complexity measure for regular languages.

So now we define our batches as in the previous section, using $k$ for the number of states. We consider a binary alphabet. Now, the question is how to determine the two factors in Eq.~\ref{eq:batches}. For the term $D_k$, we use the following result, where we provide a full proof as we did not find one in the literature: 

\begin{lemma}\label{lem:twoktwo}
If A and B are binary FSM on at most $k$ states, and $L(A) \neq L(B)$, then there exists a string $z$ of length at most $2k-2$ belonging to exactly one of these languages. Moreover, for all $k$ this bound is tight.
\end{lemma}

\begin{proof}
For tightness, see Figure \ref{example}.   
\begin{figure}[h]
\vspace{-0.3cm}
\center \includegraphics[scale=0.3]{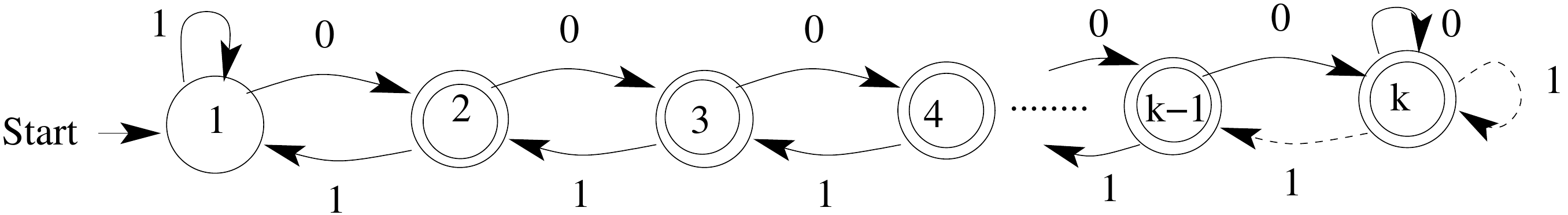}
\vspace{-0.2cm}
\caption{Two FSM A and B on $k$ states each, differing only in the dotted arrows leaving state $k$ labelled 1, with A staying in state $k$, and B going to state $k-1$. The shortest distinguishing string is $0^{k-1}1^{k-1}$, ending in state $k$ for A and state 1 for B, with state 1 being the only non-accepting state. 
On any shorter string A and B either both accept, or reject. 
}
\label{example}
\end{figure}
For the general bound, we construct a DFA C by combining A and B through a new start state $s$ and transitions from $s$ to $s_a$ on 0 and from $s$ to $s_b$ on 1, where $s_a$ is the start state of A and $s_b$ the start state of B. Final states of C are those of A and B, meaning that C accepts the language of strings $0w$ where $w \in L(A)$ and $1w$ where $w \in L(B)$. Define strings $x$ and $y$ to be C-equivalent if for any string $z$ we have $xz$ accepted by C if and only if $yz$ accepted by C. Since $L(A) \neq L(B)$ we cannot have the strings 0 and 1 C-equivalent. Thus, some string $z$ has the property that exactly one of $0z$ and $1z$ is accepted by C, meaning that exactly one of A and B accepts $z$. We now show that there exists such a string $z$ of length at most $2k-2$. We start by defining a series of equivalence relations $R_0, R_1,...,R_t$ on the state set Q of C, as follows:

\begin{itemize}
\item $R_0$ has 2 equivalence classes: the final states and non-final states of C.
\item $R_i$, for $i> 0$ is defined inductively: $R_i$ is a refinement of $R_{i-1}$, where states $q$ and $q'$ equivalent in $R_{i-1}$ are separated iff some single symbol (0 or 1) takes us, from $q$ and $q'$, to two states that are not equivalent in $R_{i-1}$.
\item $R_t$ is defined as the smallest $t$ such that $R_t=R_{t+1}$.
\end{itemize}

\noindent 
Now, if states $q$ and $p$ are equivalent in $R_t$, then no string $w$ of any length can take us, from $q$ and $p$, to two states not equivalent in $R_t$. Let us prove this statement, since it implies that states $s_a$ and $s_b$ cannot be equivalent in $R_t$. By contradiction, let $w$ be a shortest such string and also let $w$ end in 0, i.e., $w=w'0$. Assume that from $q$ and $p$, $w'$ takes us to $q'$ and $p'$, respectively. Then $q'$ and $p'$ are equivalent in $R_t$ (since $w$ was the shortest string) and we know that 0 will not take us from $q'$ and $p'$ to two states not equivalent in $R_t$, a contradiction. Thus $s_a$ and $s_b$ cannot be equivalent in $R_t$.

We construct the string $z$ distinguishing A and B by going backwards through the single symbols used to distinguish $s_a$ and $s_b$ in $R_t, R_{t-1},..., R_0$. Assume $s_a$ and $s_b$ are equivalent in $R_{i-1}$ but not in $R_i$. Then there is a symbol $c_1$ (0 or 1) that
takes $s_a$ to $s_a^1$ and $s_b$ to $s_b^1$ with these two states not equivalent in $R_{i-1}$. 
Since $s_a^1$ and $s_b^1$ not equivalent in $R_{i-1}$ there is a symbol $c_2$ (0 or 1) that
takes $s_a^1$ to $s_a^2$ and $s_b^1$ to $s_b^2$ with these two latter states not equivalent in $R_{i-2}$ (note $s_a^1$ and $s_b^1$ equivalent in $R_{i-2}$, otherwise $s_a$ and $s_b$ would not be equivalent in $R_{i-1}$).
But then the string of length 2 consisting of $c_1c_2$ take $s_a$ and $s_b$ to $s_a^2$ and $s_b^2$, two states not equivalent in $R_{i-2}$.
Continuing like this we construct a string $z=c_1c_2...c_i$ of length $i \leq t$ such that $z$ take $s_a$ and $s_b$ to two states not equivalent in $R_0$, which means that exactly one of A and B accepts $z$.

It remains to bound $t$. Note that in $R_0$ the largest equivalence class has at most $2k-1$ states, since $|Q| \leq 2k+1$ and we can assume at least two final and two non-final states (otherwise A or B is a trivial language). Since in each round we refine the equivalence relation, in $R_i$ the largest equivalence class has at most $2k-1-i$ states. Since in $R_t$ all classes have at least 1 state we note that $t \leq 2k-2$. 
%
%
%
\end{proof}

From this lemma, we can always distinguish a FSM of $k$ states from all other non-equivalent FSMs with $\leq k$ states by using the collection of all strings of length $\leq 2k-2$, i.e,. its accept/reject behaviour on this set of strings uniquely identifies it. 
There are $2^{2k-1}-1$ strings of length at most $2k-2$. Thus, a learning bias $w$ related to the number $k$ of minimal states through Equations \ref{eq:learningmonotonicity} and \ref{eq:samplingmonotonicity}
yields the (very loose) upper bound on the average biased teaching dimension of $D_k \leq 2^{2k-1}$.  
%

And now we have to choose the sampling distribution $V_k$. Since $2^{2k-1} = (1/2)4^k$, choosing $V_k=x^{-k}$, with $x > 4$, would ensure convergence, such as $x = 4 + 1/r$ with $r>0$. We must also ensure that $\sum_{k=1}^\infty V_k = 1$, which can be done e.g. by including a multiplicative factor, for example
$V_k = ((3r+1)/r)(4+1/r)^{-k}$. The actual $v$ for each different FSM (and hence concept) is just defined as $V_k / N_k$, but note that $N_k$ is not necessary
for then deriving the following bound on average \BTD:
%
\begin{align*}
&\Expected_v [\BTD_w(C)] \leq \sum_{k=1}^{\infty} V_k \cdot D_k = \nonumber \\ 
\sum_{k=1}^{\infty} ((3r+1)&/r))(4+1/r)^{-k} \cdot  (2^{2k-1}) = 2(3r+1)
\end{align*}
%
A geometric distribution with a value of $x$ greater than 4 looks like a worse result than we had for Turing machines ($x$ was around 2), but we have to clarify that the $k$ for TMs is about the length of a program, and here it is about the number of states. 
Describing an FSM of $k$ states requires a program that is exponential in $k$, based on the number of minimal such FSMs 
\cite{DomaratzkiKS02}.

\comment{
\sidenoteJose{We had some result here. Is it worth being mentioned? Is it worth being applied and convert $k$ into bits? }
Jan Arne: What we looked at earlier was a looser upper bound not thinking of minimality, and we got k^{2k}2^k different FSM on k states.
Look at it this way, and assume two symbols {0,1}:
A DFA is given by its transition function, taking a state and a symbol to a state.
The number of such functions is k^{2k} (each of the 2k state/symbol pairs can take on one of k values).
On top of this we can choose any subset of final states, thus factor 2^k extra (encoding them with final states last will just enforce that we have to consider different orderings of states as different DFAs).
Thus k^{2k}2^k for alphabet {0,1}, and in general replace 2 by alphabet size.
}

Let us briefly illustrate how the tester or teacher can design the sampling probability with a view to the usefulness of the concepts for a particular learner, within reasonable resource bounds. 
For regular languages, 
batch $C_1$ contains the two trivial concepts, either all examples positive or all negative (trivial automata accepting either no string or all strings). 
These are of little interest and with our previous sampling distribution (using the highest $r$ possible to make the distribution as entropic as possible), we would get $V_k = 3\cdot 4^{-k}$, and for $k=1$ each of the two concepts would appear with a probability $v= V_1 / N_1 = 3/8$, which seems too high.
Domaratzki et al, \cite{DomaratzkiKS02} give the (exponential) expression for $N_k$, the number of distinct minimal binary automata on $k$, which for $k=1,2,3,4$ is respectively 2, 24, 1028, 56014.
With this in mind, the teacher can select a given $k_s$ and 
give higher $V_k$ for $k < k_s$ than for $k \geq k_s$.
For testing or evaluation purposes the teacher can think that the 24 concepts in $C_2$ together should be more relevant than the two in $C_1$. 
 Furthermore, the teacher may consider that the 1028 concepts in $C_3$ are still useful but less than the earlier ones. From batch 4 and on the concepts get progressively less useful or likely, and the teacher may set $k_s=4$. However, to keep the learner alert that the whole class has to be contemplated, a non-zero probability is still assigned to each concept in batch 4 or higher, but now with a geometric distribution that deviates further from $x=4$. 
These considerations can lead to choosing a sampling probability of $V_1=1/13, V_2=8/13, V_3=3/13$ and
$V_k=1/14^{k-3}$ for $k>3$, which will ensure that $\sum_k V_k=1$.
Using $N_1=2, N_2=24, N_3=1028$ this gives a sampling probability of about
0.038 to each concept in batch 1, 0.025 to a concept in batch 2, 0.0002 to a concept in batch 3, and progressively smaller (but still non-zero) for higher batches.
With this choice of $v$ the expected \BTD becomes less than 22, as shown below, and if the teacher has more (or less) resources at hand she may alter the sampling probability accordingly. 
\begin{align*}
\Expected_v [\BTD_w(C)] =
\sum_{k=1}^{\infty} D_k \cdot  V_k = \sum_{k=1}^{\infty} (1/2)4^{k} \cdot  V_k = \\
 2 \cdot 1/13 + 8 \cdot 8/13 + 32 \cdot 3/ 13 + 12.8 < 22
\end{align*}
%


\section{Discussion}\label{sec:connection}

Analysing whether and how infinite concept classes can be taught led us to a dilemma between making the teaching set finite on average and the use of a wide, entropic sampling distribution actually covering the whole class. The observation that humans are able to cover a wide range of concepts and can learn from very few examples suggests that humans share a strong bias and may communicate, and teach, accordingly. The strong bias may well depend on the application, domain or context, but it can also be based on the complexity of the concept, as we have investigated here, very much in the same way to other theories of inductive inference such as Solomonoff's prediction, the use of Occam's razor, structural risk minimisation or the MML/MDL principles \cite{solomonoff1964formal1,Wallace-Boulton1968,
LV2008}. 
Therefore, we can think of this work as bringing the above setting from the standard learning scenario to the teaching scenario, with further connections to be unveiled with possibly more positive results. 
In practice, these ideas have worked well for learning from very few examples in areas such as inductive programming, programming by examples or teaching by demonstration \cite{gulwani2015inductive,gulwani2016programming,Ho2016,shu2017neural}, usually without recognising the two different biases involved.


The notion of simplicity for TMs depends on the choice of the UTM. 
Similarly, for FSMs, the number of states is a natural measure of simplicity, but others could be used, such as the length of the shortest regular expression expressing the concept. The invariance theorem \cite{LV2008} establishes that simplicity is the same up to a constant that is independent of the concept, but this constant can be large. This motivates a possible study of other versions of the \BTD, more independent from the particular complexity measure. For instance, the \BTD could be modified in such a way the concept is only identified when the posterior probability reaches a certain level. 
This could be compared to the analysis of all concepts of size $k + m$, with some margin $m$. 


Another interesting thing to analyse is to consider $k$ as a measure of difficulty of the concept and consider the session as an evaluation process. In this case, the sampling distribution could be adapted in such a way that, if we know the ability of the learner, we could sample concepts of appropriate complexity $k$. In other words, the sample distribution could assign very low probability to the very easy examples (small $k$) but still (necessary) decreasing probability from some given $k$, resembling a Poisson distribution, and breaking the monotonicity of Eq.~\ref{eq:samplingmonotonicity}. 

The analysis of complex concept classes is sometimes avoided because positive results are elusive. Here, the very notion of expected teaching dimension forces us to consider non-uniform distributions. 
This work has made clear that a trade-off is necessary between an effective teaching and a wide coverage of the concept class. This gives several insights about how biases have to be embedded and used by learner and teacher, and also suggestions about efficient concept understanding and communication in general. 
 


\subsubsection*{Acknowledgments}
{\small This work has been partially supported by the EU (FEDER) and the Spanish MINECO under grants TIN 2015-69175-C4-1-R and by Generalitat Valenciana under grant PROMETEOII/2015/013. This work was done while the second author was visiting the Universitat Polit\`ecnica de Val\`encia during the schoolyear 2016-17.}


\bibliography{biblio}

\end{document}